\documentclass[sigconf]{acmart}

\usepackage{booktabs} 
\usepackage{balance}

\setcopyright{none}

\acmConference[2018]{}{Oct 2018}{} 
\acmYear{2018}
\copyrightyear{2018}

\acmPrice{15.00}

\usepackage{url}
\usepackage{amssymb}
\usepackage{tikz}
\usepackage{amsmath}
\usepackage{float}
\usepackage{wrapfig}
\usepackage{amsthm}

\usepackage{subfigure}

\usepackage{placeins}

\restylefloat{table}

\usepackage{multirow}
\usepackage{algorithm}
\usepackage[noend]{algpseudocode}

\usepackage{balance}

\newcommand{\eg}{{e.g.}}
\newcommand{\ie}{{i.e.}}

\renewcommand{\G}{\mathcal{G}}
\newcommand{\red}{\prec}

\usepackage{xspace}
\newcommand{\pattern}{CQC pattern\xspace}
\renewcommand{\patterns}{CQC patterns\xspace}

\newcommand{\software}{MiSPa\xspace}

\theoremstyle{plain}

\newtheorem{dft}{Definition}[section]

\begin{document}
\title{Mining Contrasting Quasi-Clique Patterns}

\author{Roberto Alonso }
\affiliation{%
  \institution{Technical University of Munich}
}
\email{alonsor@in.tum.de}

\author{Stephan G\"unnemann}
\affiliation{%
  \institution{Technical University of Munich}
}
\email{guennemann@in.tum.de}


\begin{abstract}
Mining dense quasi-cliques is a well-known clustering task with applications ranging from social networks over collaboration graphs to document analysis. Recent work has extended this task to multiple graphs; i.e.~the goal is to find groups of vertices highly \emph{dense} among multiple graphs. In this paper we argue that in a multi-graph scenario the \emph{sparsity} is valuable for knowledge extraction as well.
We introduce the concept of {contrasting quasi-clique patterns}: a collection of vertices highly dense in one graph but highly sparse (i.e.~less connected) in a second graph. Thus, these patterns specifically highlight the difference/contrast between the considered graphs.
Based on our  novel model, we propose an algorithm that enables fast computation of contrasting patterns by exploiting intelligent traversal and pruning techniques.
We showcase the potential of contrasting patterns on a variety of synthetic and real-world datasets.
\end{abstract}
%

%
%

\ccsdesc[500]{Mathematics of computing~Graph theory}
\ccsdesc[300]{Mathematics of computing~Graph algorithms}
\ccsdesc[100]{Mathematics of computing~Graph enumeration}
\ccsdesc[500]{Information systems~Clustering}

\keywords{quasi-cliques, contrasting patterns, graph mining}

\maketitle

\section{Introduction}

The World Wide Web, social networks, and e-commerce platforms have been successfully studied using various graph mining methods. One of most prominent tasks is mining quasi-cliques, i.e. detecting sets of vertices that are \emph{densely} connected in a graph. 

While traditionally only a single graph was considered, recent research  \cite{pjz05, crossgraph3, zwzk06, mimag} has extended the search for quasi-cliques to the setting where multiple graphs are given (also known as multi-layer/multi-dimensional graphs). Often these graphs represent different types of connections (collaborations, similarity, citations, etc.). Here, the goal is to find groups of vertices that are densely connected among multiple graphs at the same time (cross-graph quasi-cliques); which has proven useful in several applications, most prominently in community detection (see \eg~\cite{Kim15}). While all existing work in the area of cross-graph quasi-cliques focused on detecting \emph{dense} areas only, in this work we argue that for the multi-graph scenario \emph{sparsity} is valuable as well. 

Consider, for example, a graph of products in an e-commerce system where on one hand connections indicate co-purchases (graph~1), and on the other their textual similarity w.r.t their description (graph 2). Products highly connected in the co-purchasing graph but textually dissimilar are \textit{complementary} products; \eg~socks and shoes are purchased together but have different description (see Fig.~\ref{fig:model} for an illustrative example). In contrast, similar products not being purchased together are \textit{substitute} products; \eg~the description of shoes is similar, but often only one model is purchased. Finding complementary and substitute products is useful to make meaningful product suggestions and improve the overall shopping experience~\cite{McAuley15}.

As another example, consider a graph of documents (\eg~scientific papers, websites, etc) where one graph (graph 1) represents citations between documents\footnote{Document A cites B or B cites A;~\ie~an undirected graph}, and on the other their textual similarity (graph 2). Intuitively, a group of similar documents should be highly connected in the citation graph as well; \eg~papers about quasi-cliques are textually similar and often have a comparable state of the art. Thus, finding a group that is densely connected regarding their content, but \emph{sparsely} connected regarding their citations is useful to detect, \eg, plagiarism or to identify missing connections between research works. We will show some real-world examples of these patterns in Sec.~\ref{exp:case}.

\begin{figure}[t]\centering
	\includegraphics[scale = 0.34]{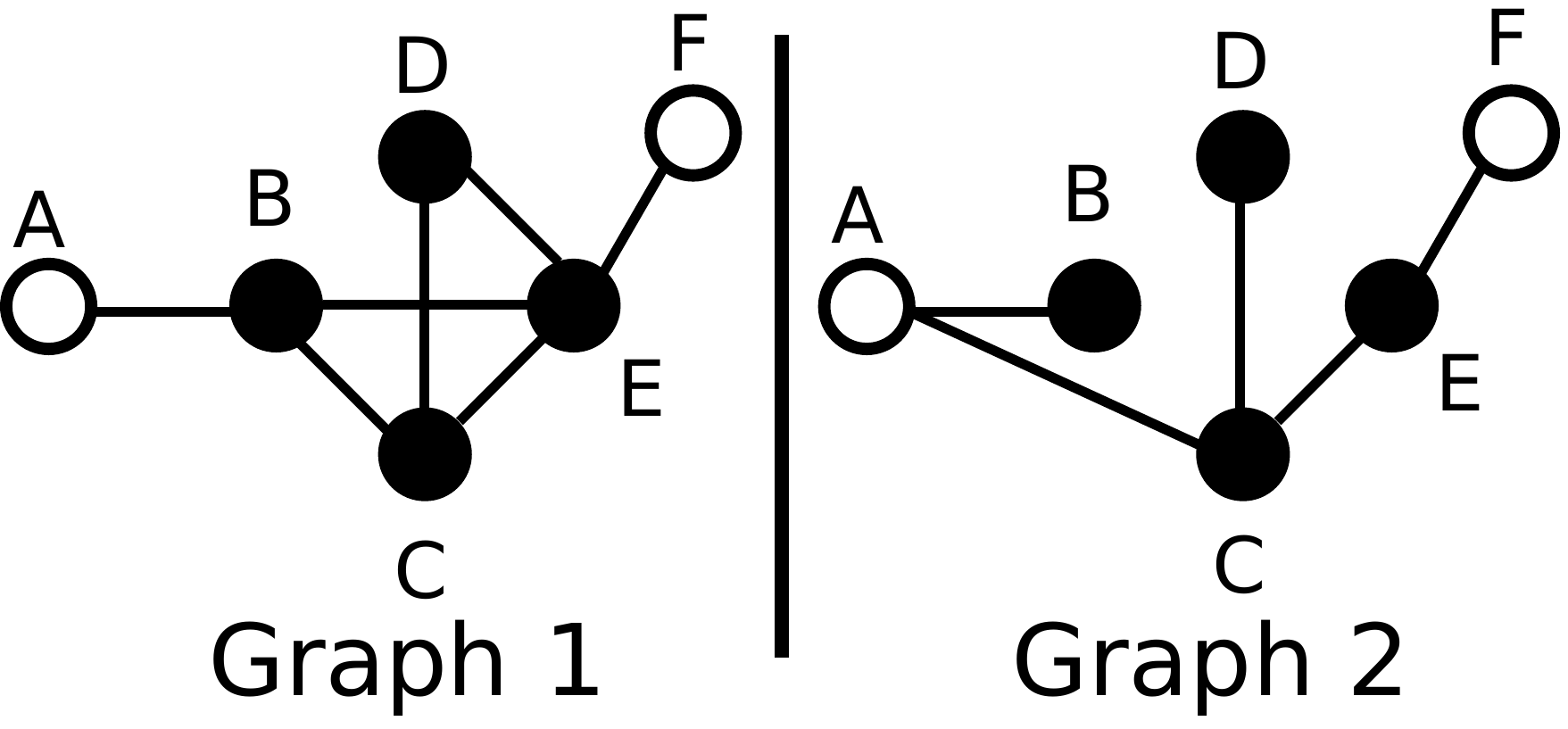}
	\vspace*{-1mm}\caption{Contrasting quasi-clique pattern: The vertex set $\{B,C,D,E\}$ is dense in graph~1 (e.g. co-purchasing graph) but sparse in graph 2 (e.g. similarity graph).}
	\label{fig:model}
	\vspace*{-2mm}
\end{figure}

Now, consider a group of people, one graph (graph 1) represents their day-to-day contact\footnote{E.g.~given by sensors, surveys or GPS coordinates} and the other (graph 2) their friendship on Facebook. People with whom we have a lot of interaction (\eg~friends or family) are expected to be Facebook contacts;~\ie~dense in both graphs. However, people highly connected in the Facebook graph and at the same time sparsely connected in the contact graph is an interesting pattern to give insights for a social or psychological study, while people highly connected in the contact graph with few connections in the Facebook graph might represent a different type of relationship;~\eg~in principle, teachers, bosses, and supervisors are not expected to be Facebook friends. Later in Sec.~\ref{exp:case} we show a real case study showing interactions between High School students as given by~\cite{Mastrandrea16}.



As these examples indicate, in a multi-graph scenario sparsity becomes an important characteristic. 
Based on this motivation, this paper introduces the concept of \emph{contrasting quasi-clique patterns} (\patterns): a collection of vertices highly dense in one graph, and simultaneously sparse in a second graph.\footnote{It is important to note that finding 'sparse' quasi-cliques in a \emph{single} graph, is (usually) not a reasonable task: any set of disconnected vertices forms a group with density of zero. Due to the general sparsity of graph data, however, this non-connectedness is not very surprising in general. However, in the case of \emph{multiple} graphs, sparsely connected vertices become interesting -- if their density in the second graph is much higher. Thus, the \emph{difference} in the density values is important.}

In our model, we consider the fact that a vertex can naturally belong to more than one pattern. Thus, we allow our patterns to overlap. 
However, allowing overlap often leads to the problem of redundancy since many similar patterns might be generated \cite{rescu,GFB+10}. 
Thus, based on our novel pattern definition, we introduce a model for generating a set of patterns that are non-redundant, though, still allows the patterns to overlap to a certain extend. This final set of patterns should contain the \emph{most contrasting patterns}.

Finally, since determining the overall clustering according to our model is NP-hard, 
we introduce an approximate algorithm to compute it. 
 The idea is to use a best-first principle to explore a joint enumeration tree in an informed fashion; thus, analyzing both graphs simultaneously and starting to enumerate the most contrasting patterns first. 
Our experimental analysis shows that this specialized approach is much more efficient than relying on existing \emph{dense} cross-graph quasi-clique detection methods.\footnote{By applying these techniques on one graph's complement one could, in principle, detect sparse patterns. This approach, however, does not scale to large graphs.}

 
\noindent The contributions of this work are as follows: \vspace*{-0.5mm}
\setlength{\leftmargini}{14pt}
\begin{itemize}\setlength{\itemsep}{0pt}
\item We introduce the problem of finding contrasting quasi-cliques patterns (\patterns). 
\item We propose the algorithm \software to approximate the problem of finding \patterns.
\item We perform experiments on several real-world datasets confirming the potential of our novel pattern definition.

\end{itemize}
We want to highlight that our task is different to discriminative subgraph discovery and to community detection in signed networks (see.  Sec.~\ref{sec:relatedWork}).

\textbf{Overview.} Sec.~\ref{sec:relatedWork} presents related work. Sec.~\ref{sec:patternModel} introduces the \pattern model, followed by our \software algorithm in Sec.~\ref{sec:algorithm}. Sec.~\ref{sec:results} shows our experiments, and Sec.~\ref{sec:conclusions} concludes the~paper. 

 \section{Related Work}
\label{sec:relatedWork}

Multiple mining tasks for graphs exist \cite{aggarwal2010managing}. Relevant to this work is the field of  ``dense subgraph mining''. Here, different models and notions have been introduced, with cliques and 
$\gamma$-quasi-cliques \cite{pjz05,DBLP:conf/pkdd/LiuW08} being the most prominent ones.

Several clustering approaches have been proposed which take a set of networks as input, where each network represents a particular kind of relationship. 
This type of data if often called \emph{multi-relational network} \cite{cai2005community}.
The work \cite{DBLP:conf/ismis/CerfNB09}, for example, aims to find cliques in subsequent time steps in a dynamic graph represented as a 3-dimensional boolean cube. By basing on the algorithm of \cite{DBLP:journals/tkdd/CerfBRB09}, the desired cliques are specified by (anti-)monotonic constraints. 
Since it exploits the (anti-)monotonic property of cliques, 
an extension to quasi-cliques is not possible as the quasi-clique model does not fulfill a monotonicity property. 

In \cite{pjz05}, the principle of cross-graph quasi-cliques detection has been introduced.
Given a database of graphs each having the same vertices, a cross-graph quasi-clique is defined as a set of vertices that forms a quasi-clique in \emph{all} of the graphs. Only maximal sets having this property are output. 
Similarly, the approaches \cite{crossgraph3} and \cite{zwzk06}  mine sets of vertices that form a clique \cite{crossgraph3} or quasi-clique \cite{zwzk06} in at least a certain percentage of the graphs in the database.
Both approaches aim at mining \emph{closed} (quasi-)cliques. Last, in \cite{mimag} a method for finding non-redundant quasi-cliques spanning across multiple graphs has been proposed. 

All of the above works focused on finding dense quasi-cliques only. In contrast, our work aims at finding contrasting patterns: groups of vertices that are dense and sparse at the same time.

As a naive approach to detect contrasting patterns one could
consider to use one of the above dense subgraph mining algorithms
when operating on the complement graphs. We compare with such
an approach in our experimental study.

Finally, please note that the task of community detection in signed networks \cite{anchuri2012communities,esmailian2015community}, i.e.\ networks with positive and negative edges, is different to our principle. There, one is still interested in finding dense clusters, with edges of different kind. Negative edges, however, do not mean sparsity -- which is the focus of this work. Indeed, signed networks and our principle could be combined. 

Also note the difference to the principle of discriminative subgraph mining \cite{thoma2010discriminative,jin2011lts,ting2006mining} (rarely also called contrast subgraphs as in \cite{ting2006mining}), a supervised-learning task. Here a database of multiple (usually attributed) graphs is given, each belonging to a certain class we aim to predict. The goal is to find subgraphs that discriminate between the classes, i.e. subgraphs appearing in many graphs of one class but only few graphs of the other classes. Whether these subgraphs are dense or sparse is not relevant.

\section{Pattern Model}
\label{sec:patternModel}


This section formalizes the problem of finding contrasting quasi-clique patterns (\patterns); a collection of vertices dense in one graph but sparsely connected in the other. To this aim we appeal to the notion of quasi-cliques, a set of vertices which is almost completely connected. 

There have been multiple works on quasi-cliques which rely primarily on two formulations. We will show that choosing a single quasi-clique formulation to detect \patterns has severe drawbacks. Thus, we combine both definitions in a single measure of interestingness.

The input under consideration is a set of graphs $\G = \{G_1,G_2\}$ where each graph $G_i = (V,E_i), E_i \subseteq V \times V$ is an undirected graph without self-loops. Both graphs share the same vertices but with different edges. Note that it is possible to consider different vertex sets $\{V_1,V_2\}$ by simply using $V = \bigcup V_i$. 


\subsection{Existing quasi-clique models and their limitations to detect \patterns}
\label{sec:quasicliques}

Before specifying our new model, we first introduce the existing quasi-clique models and highlight their limitations for \patterns detection. The first definition of quasi-cliques as, e.g., given by \cite{mimag,DBLP:conf/pkdd/LiuW08,zwzk06} is:

\begin{dft}[$\delta$-Quasi-clique]\label{def:clique}~\\Given a vertex set $O \subseteq V$ in a graph $G_i = (V,E_i)$, and $\delta \in (0,1]$, $O$ is a $\delta$-quasi-clique if\\[-2mm] $$\forall v \in O: deg^O_{G_i}(v) \geq \lceil \delta \cdot (|O| -1) \rceil$$ where $deg^O_{G_i}(v) = |\{u \in O \mid (u,v) \in E_i\}|$.
The density  of a quasi-clique $O$ in graph $G_i$ is defined by\\[-2mm] $$\textstyle \gamma_{G_i}(O) = \frac{\min_{v \in O}\{deg^O_{G_i}(v)\}}{|O|-1}$$
\end{dft}

The higher $\delta$, the higher the required density of the quasi-clique. For $\delta=1$, the definition represents cliques. 

The second definition of quasi-cliques\footnote{With respect to this definition sometimes called pseudo-cliques or cohesive patterns.}, inspired by works like~\cite{Uno2010, copam}, is defined as follows:

\vspace*{1mm}\begin{dft}[$\theta$-Quasi-clique]~\\Given a \emph{connected} vertex set $O \subseteq V$ in a graph $G_i = (V,E_i)$, and $\theta \in (0,1]$, $O$ is a $\theta$-quasi-clique if$$\alpha_{G_i} (O) \geq \theta$$ 
where $\alpha_{G_i} (O) = 2|E_i(O)|/(|O|\cdot(|O|-1))$ is the proportion of observed edges $E_i(O)$ in $O$ w.r.t graph $G_i$ compared to the potential edges in a (complete) clique, called simply density.

\end{dft}

The crucial difference between both definitions is that in a $\delta$-quasi-clique \emph{each} vertex has to be sufficiently connected, while in a $\theta$-quasi-clique only the average density is important. This seemingly small difference has a severe effect when trying to find dense or sparse regions in a graph.

1) When considering dense regions, the $\delta$-quasi-clique definition favors the desirable detection of ``tightly and relatively evenly'' connected quasi-cliques as shown by \cite{zwzk06} -- i.e. the density is relatively homogeneous in the pattern. In graph 2 of Figure~\ref{fig:example_pattern}, the set $O=\{A,B,C,D\}$ is a valid $\delta$-quasi-clique, while the set $O'=\{A,B,C,D,E\}$ is none. This property matches the intuition, since the node $E$ has only one edge, making the pattern unevenly connected and not helpful for the density of the group.

In contrast,  using $\theta$-quasi-clique the set $O'$ would still be regarded as relatively dense ($\alpha= 7/10$). Thus, $\delta$-quasi-cliques are better suited for finding homogeneous dense regions.

2) When considering sparse regions, the roles switch: Since the definition of $\delta$-quasi-cliques enforces \emph{each} vertex to be sufficiently connected, \emph{a single} less-connected vertex would lower the density significantly. In Figure~\ref{fig:example_pattern}, the density of the vertex set $O'=\{A,B,C,D,E\}$ in graph 2 is $\gamma = 0.25$, \ie~seemingly sparse -- while obviously the region is not entirely sparse.  By contrast, the $\theta$-quasi-clique definition does not suffer from this drawback.



While clearly each definition of density and sparsity has its own characteristics, we argue that the resulting patterns should be relatively homogeneous regarding their density and sparsity. That is, the density/sparsity should not just be the effect of individual nodes. 

Accordingly,  since our goal is to find highly dense vertices in one graph which are sparsely connected in the other, we combine the rationale behind $\delta$ and $\theta$ quasi-cliques. While $\delta$ is more suited to capture density, $\theta$ is better to capture less connected vertices.

\begin{figure}[t]
	\includegraphics[scale = 0.44]{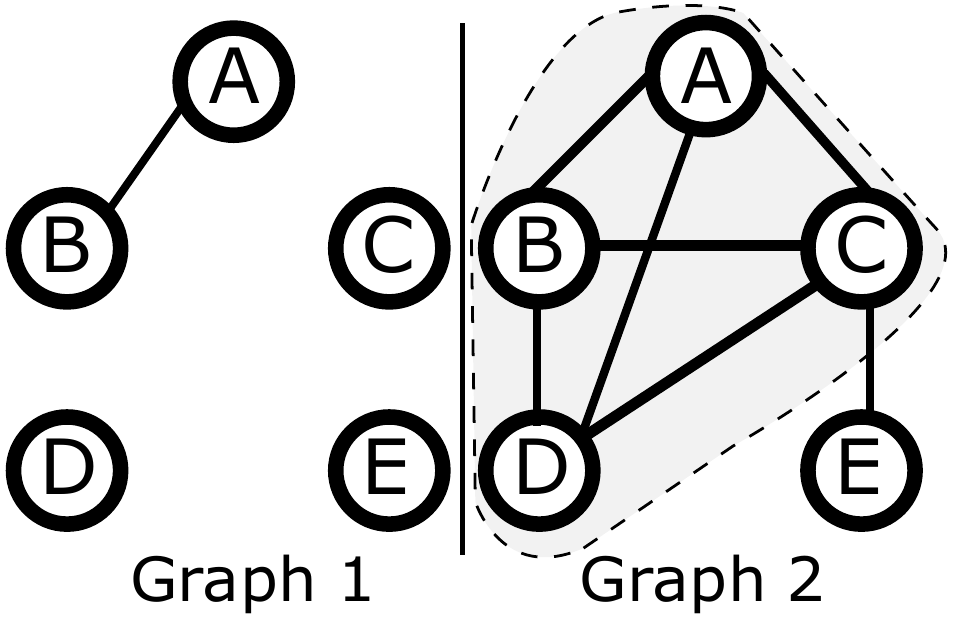}
	\vspace*{-1mm}\caption{Example to illustrate our \pattern model}
	\label{fig:example_pattern}
\end{figure}



\subsection{\patterns}

More precise: First, we ensure that the nodes in one graph are sufficiently connected using the $\delta$-quasi-clique model. Second, to define sparsity -- or more precise our measure of contrast -- we borrow from the $\theta$-quasi-clique definition by referring to the density measure $\alpha$;~\ie~the fraction of observed edges compared to the potential edges:  If the fraction is high in one graph but low in the other, we observe a significant difference in the number of edges. Thus, the contrast of their density values is high. 

Formally:
\vspace*{1mm}\begin{dft}[$(\delta,\delta')$-contrasting quasi-clique]~\\
	Given a vertex set $O \subseteq V$, graphs $\G = \{G_1,G_2\}$,  $\delta \in (0,1]$, and $\delta'\in (0,1]$.
	$O$ is a $(\delta,\delta')$-\pattern if \\[-4mm]
	\begin{itemize}
			\item $\max\{{\gamma}_{G_1}(O),{\gamma}_{G_2}(O)\}\geq \delta$~~~// i.e.  $O$ is a $\delta$-quasi-clique in at least one graph
		\item $c(O) > \delta'$~~~// i.e. the contrast in densities is high
	\end{itemize}
where $c(O):= |\alpha_{G_1}(O) - \alpha_{G_2}(O)|$ is the contrast of the pattern.

\end{dft}

To illustrate better \patterns consider Fig~\ref{fig:example_pattern}, and the set $O=\{A,B,C,D\}$. In the second graph, $O$ is a quasi-clique of high density ($\gamma = 1$); in contrast, in the first graph, the density is zero since $C$ and $D$ are disconnected. The contrast value is given by $6/6-1/6=(6-1)/(0.5\cdot 4 \cdot 3)=0.83$. Thus, $O$ is a $(1, 0.83)$-\pattern. The set $O'=\{A,B,C,D,E\}$ is not a \pattern; $O'$ is not a $\delta$ quasi-clique in any of the graphs. 

\subsection{Selection of $\delta$ and $\delta'$} 

The \patterns can be controlled by selecting $\delta$ and $\delta'$. A low value of $\delta$ produces sparse patterns while a large value dense ones. Also, if we set $\delta'$ to a large value we get patterns with more difference in the number of edges -- or more precise a high contrast -- while a small $\delta'$ favors the detection of patterns with few contrast. Thus, by combining both parameters into a single definition (\ie~\patterns), we should consider the influence that both have in the final result. As an example, a small $\delta$ and a large $\delta'$ (\ie~sparse vertices with high contrast) might produce no patterns at all wasting computational resources. By contrast, if we set $\delta$ to a large value and $\delta'$ to a small one we get highly dense patterns with contrast. 

Further, notice that a \pattern pattern has to be, first, a potential quasi-clique in at least one of the graphs before considering it a contrasting pattern. Thus, $\delta$ heavily influences the number of potential patterns to investigate while $\delta'$ refines the enumeration of these by means of an upper bound on the interestingness of the patterns as we shall show below (in Section~\ref{sec:quality_estimation}).

In general, we suggest to fix $\delta' = 0$ (\ie~we will be interested in all patterns with contrast) and set $\delta \geq 0.5$, since in this case the nodes are tightly connected and guaranteed to be connected \cite{zwzk06}. 








\subsection{Overall pattern result}
\label{sec:clusteringmodel}

Multiple sets of vertices can fulfill the definition  of \patterns, thus, potentially leading to a large amount of patterns that are quite similar -- and, therefore, not interesting for the user.
In Fig.~\ref{fig:redundant}, e.g., the set $O=\{A,B,C,D\}$ and $P=\{A,C,D\}$ are both reasonable \patterns; however, they capture almost the same information.

\begin{figure}[h]
	\vspace*{-1mm}
	\includegraphics[scale = 0.41]{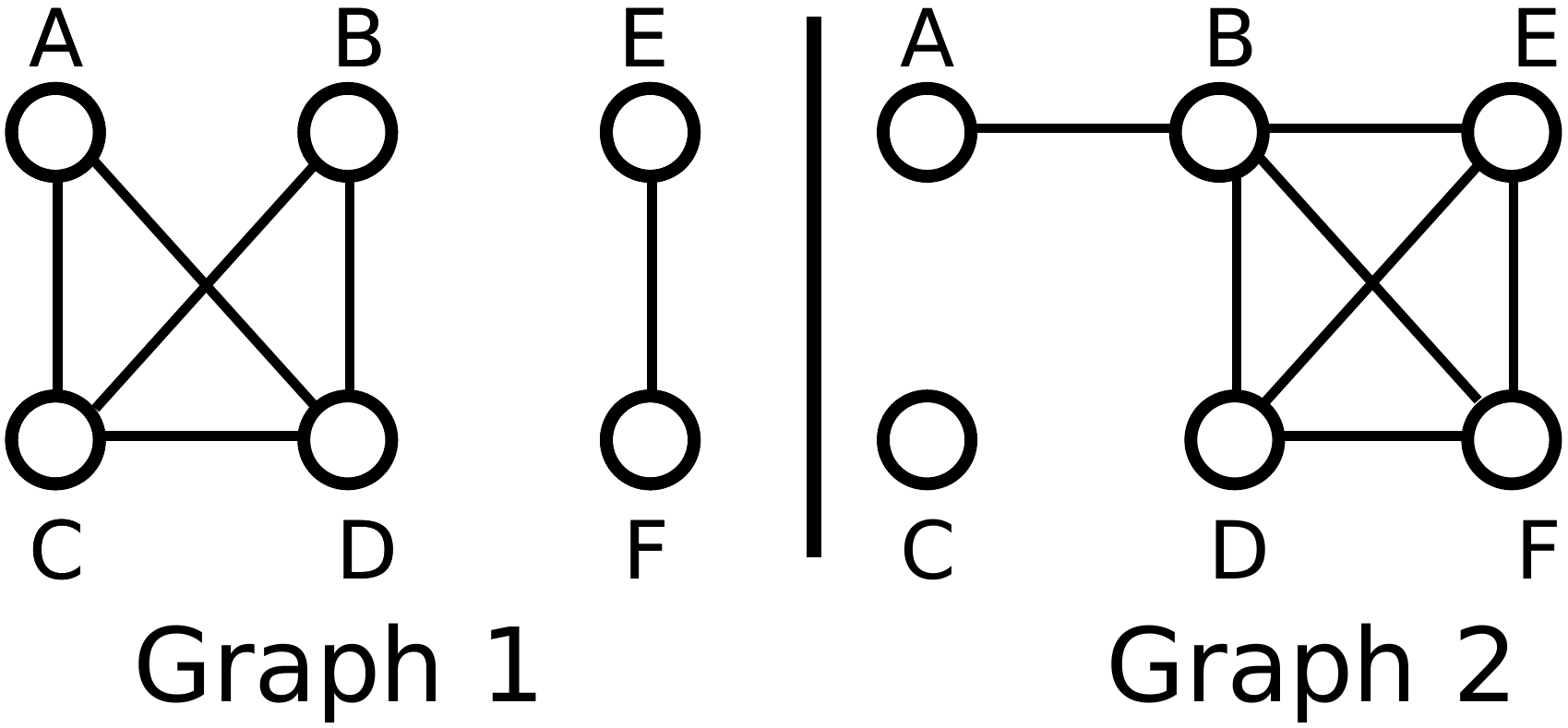}
	\vspace*{-2mm}\caption{Example to illustrate redundant \patterns}
	\label{fig:redundant}
\end{figure}

To refine the output to a smaller subsets of patterns, we consider only the \emph{most interesting}, \textit{non-redundant} ones.

What are the interesting patterns? We argue that two aspects are of primary importance for our task: First, patterns showing a high contrast are interesting. Observing a set of vertices that is completely connected in one but completely unconnected in the other graph is a surprising result.
 Second, patterns of large size are interesting. This matches the idea of classical quasi-clique approaches where primarily maximal patterns are detected. We here simply combine both measures into a single \textit{interestingness} function\\[-2mm]
\begin{equation}\label{eq:inter}
\textstyle I(O) := |O|\cdot c(O)=\frac{2\cdot \text{abs}(|E_1(O)|-|E_2(O)|)}{|O|-1}   
 	\end{equation}
Furthermore, patterns which are not 0.5 quasi-cliques in at least one of the graphs and which are smaller than 4 vertices are not interesting;~\ie~we set $I(O)=-1$ for these patterns.




What are redundant patterns? To define redundancy, we  adapted the redundancy relation as described in~\cite{mimag}. Intuitively, a pattern $O$ is redundant to a pattern $P$, if it is less interesting and a high fraction of $O$'s edges are already covered by $P$. Formally,  
\begin{dft}[Redundancy] \label{def:redundancy}
	Let $O$ and $P$ be two \patterns. $O$ is redundant to $P$ (short: $O \red P$) if\\[-4mm]
	$$O \neq P \wedge I(O) \leq I(P) \wedge \frac{1}{2} \sum_{i \in \{1,2\}} \frac{|E_i(O) \cap E_i(P)|}{|E_i(O)|} \geq r$$
	~\\[-3mm]for the redundancy parameter $r \in (0,1]$.
\end{dft}

Consider again Fig.~\ref{fig:redundant}: The set $O=\{A,B,C,D\}$ has an interestingness of $I(O)=4\cdot 0.5=2$, the set $P=\{A,C,D\}$ of $I(P)=3\cdot 1=3$, and the set $Q=\{B,D,E,F\}$ of $I(Q)=4\cdot 0.66=2.66$. In this example $P$ is more interesting than $O$ since $P$ is not connected at all in the second graph. Thus, $O$ is redundant to $P$ since it is less interesting and they share most of their edges. $Q$ is not redundant to the other patterns since it covers different edges.

Given the interestingness function and the redundancy relation, we define the final output as the result of a constrained combinatorial optimization problem. We aim to find a set of \patterns that are pairwise non-redundant and together maximize the interestingness. Formally,
\begin{dft}\label{eq:problem}
	The final result of patterns is given by\\[-2mm]
$$ \mathcal{R}^*=\arg\max_{\stackrel{\mathcal{R}\subseteq \mathcal{D}}{\neg\exists O,O' \in \mathcal{R}:O\red O'}}\Sigma_{O\in \mathcal{R}} I(O)$$
where $ \mathcal{D}:=\{ O\subseteq V \mid O \text{ is a $(\delta,\delta')$-\pattern and $I(O)>0$}  \} $ is the set of all interesting \patterns

\end{dft}

\subsection{Complexity of Finding \patterns}\label{sec:complexity}
The problem of finding \patterns is NP-hard. This can be shown by reducing the NP-hard quasi-clique detection problem \cite{DBLP:conf/pkdd/LiuW08} to our problem: Assume we want to find (or even enumerate) quasi-cliques in a graph $G=G_1$. By constructing a second graph $G_2$ without any edges, the detection of $(\gamma,0)$-contrasting quasi-cliques in $\{G_1,G_2\}$ corresponds to finding $\gamma$-quasi-cliques in $G_1$.
Hence, our problem is NP-hard, and requires an approximate solution for efficiency. 

\section{Algorithm \software}
\label{sec:algorithm}

To compute \patterns one approach is to use an efficient algorithm to find quasi-cliques (\eg~\cite{mimag,copam,zwzk06,DBLP:conf/pkdd/LiuW08}) in each individual graph first. Then for each cluster one can compute the density of the clusters with respect to the other graph to check whether a high contrast is obtained. Finally, the most interesting and non-redundant patterns can be selected in a post-processing step.

This approach has two severe limitations: First, one might generate a large set of uninteresting patterns that are dense in both graphs. Thus, only wasting computation time. On the other hand, one might miss important patterns. Since the existing quasi-clique approaches are steered towards finding dense subgraphs only, slightly less dense subgraphs might not be reported 
-- these subgraphs, however, might show much stronger contrast values regarding the second graph.

Hence, we have designed the \software algorithm to efficiently identify \patterns. Our algorithm is inspired by the MiMAG algorithm~\cite{mimag} that simultaneously analyzes multiple graphs. The huge difference being that our algorithm is steered towards finding high contrast patterns.

Based on Sec.~\ref{sec:complexity}, we cannot expect to find an efficient algorithm computing an exact result. Thus, instead of determining a result with \emph{maximum} interestingness, we compute a \emph{maximal} result. That is, a result where patterns are pairwise non-redundant, have \emph{high interestingness}, and adding any further pattern would lead to redundancy.

The core idea we follow is twofold: (i) instead of analyzing the graphs individually, we traverse them jointly, and (ii) we traverse the graphs in a informed fashion to enumerate the most interesting patterns first. For this purpose, we adapt the traversal principle proposed in \cite{DBLP:conf/pkdd/LiuW08,mimag}. 

\subsection{Joint enumeration tree}
In our algorithm, vertex sets $O\subseteq V$ are enumerated by a  traversal in the \emph{set enumeration tree} \cite{DBLP:conf/kr/Rymon92}.\footnote{\small To avoid confusion, we use  ``vertex'' for a vertex in the original graph and ``node'' for the nodes of the set enumeration tree, which represent \emph{sets} of vertices.}
An example tree for a graph with three vertices is shown in Figure~\ref{fig:setenumtree}. 
Basically, the set enumeration tree contains all possible vertex sets $O \subseteq V$. 
Each set visited by the traversal of the tree is a potential \pattern.

\begin{figure}[h!]
	\vspace*{-1mm}
	\includegraphics[scale = 0.42] {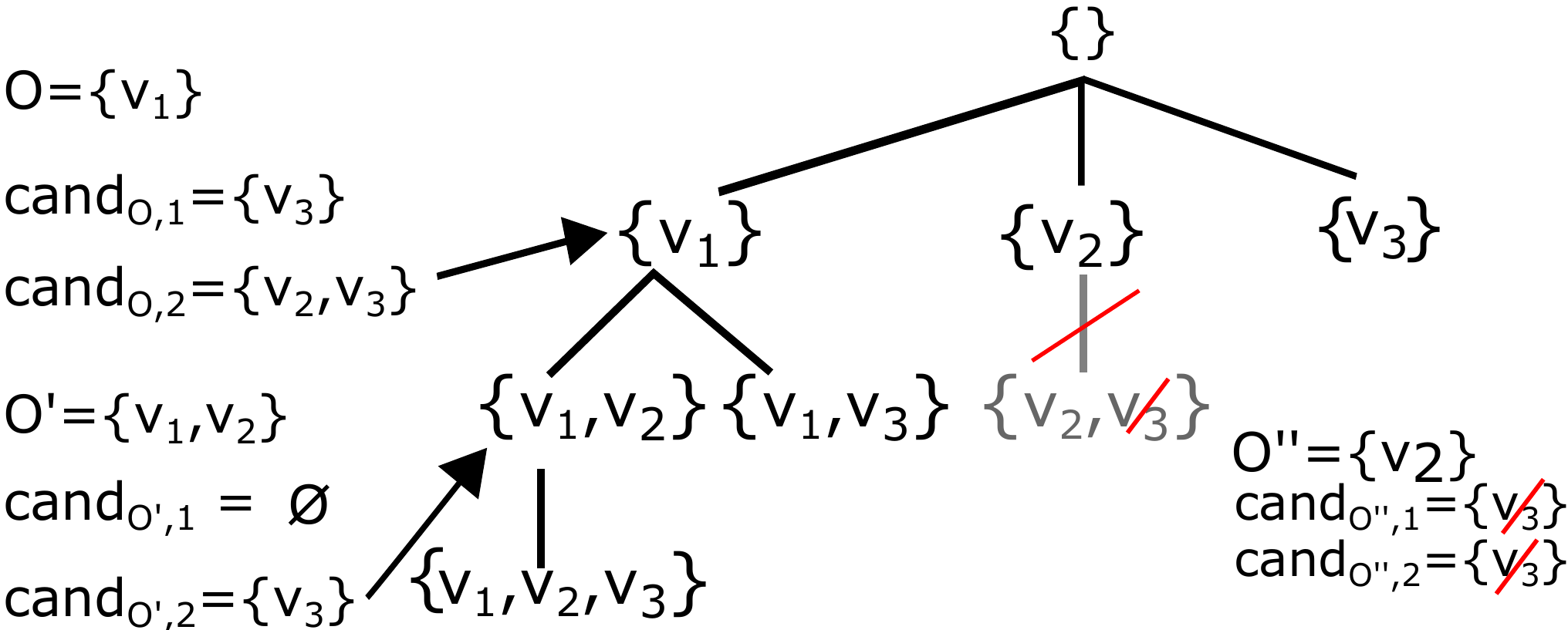}
	\vspace*{-1mm}\caption{Joint set enumeration tree}
	\label{fig:setenumtree}
\end{figure}

The crucial aspect is to prune this enumeration tree, i.e. to avoid traversing it completely. For this purpose, each node $O$ is associated with two candidate sets $cand_{O,1}$ and $cand_{O,2}$. The candidate set $cand_{O,i}$ contains all vertices that can (potentially) be added to $O$ to form a quasi-clique in graph $i$ (one property of our \patterns definition). Note that these two candidate sets might (and will primarily) be different since our goal is to find patterns $X\supseteq O$ that are dense in one graph but sparse in the other.

To lower the cardinality of a candidate set, we apply multiple pruning methods \cite{DBLP:conf/pkdd/LiuW08}. That is, if $O$'s candidate set $cand_{O,i}$ contains a vertex $v$ that can never be part of a quasi-clique $O' \supset O$ in graph $i$, we can delete $v$ from the candidate set. If a vertex $v$ can be removed from \emph{both candidate sets}, the whole subtree rooted at the current node that contains $v$ can be pruned. In none of the graphs, the vertex can be part of a quasi-clique $X\supseteq O\cup \{v\}$.
As an example, consider Figure~\ref{fig:setenumtree} and assume $v_3$ is not a promising vertex w.r.t the node $O''=\{v_2\}$, then $v_3$ can be deleted from both $cand_{O'',1}$ and $cand_{O'',2}$. Lastly, since both candidate sets are empty, the subtree rooted at $\{v_2\}$ can be safely pruned as it can never be part of a quasi-clique.

\subsection{Informed tree traversal}
The above pruning principle lowers the area of the search space that needs to be investigated. However, still too many patterns might be required to analyze. Therefore, in a next step, instead of following an uninformed (e.g. depth-first) traversal of the tree as in \cite{DBLP:conf/pkdd/LiuW08}, we perform a best-first (A*-like) traversal.

To realize this informed search, we need to provide bounds / estimates for the interestingness of the patterns expected in a subtree. For each node $O$, we compute a estimation that provides an \emph{upper bound} for the maximal interestingness of any pattern that can be found in the subtree rooted at $O$. 
Using this bound, we traverse the most promising subtrees first. In Section \ref{sec:quality_estimation}, we show how to compute this bound.

Given this informed search, for each set $O$ we visit during the traversal, we check if $O$ forms a valid \pattern. If so, we add it to an intermediate priority queue -- we cannot add it to the final result yet since there might be further more interesting patterns in other subtrees (the estimated interestingness provides an \emph{upper bound} of the \emph{subtree}'s interestingness only; thus, $O$ itself might have a lower interestingness). This priority queue  contains the set of subtrees that are still to process, as well as the set of already detected patterns that are not added to the result so far. This queue is sorted by the (estimated) interestingness values of the subtrees and patterns.

\subsection{Overall Processing Scheme}
\label{sec:overallScheme}
Based on these ideas, the overall processing of \software is shown in Algorithm \ref{pseudocode}.  Given the two input graphs, \software computes a maximal set of \patterns which contains only non-redundant patterns.
Initially, the set $Result$ is empty (line \ref{line:resultempty}); it will be iteratively filled during the processing. 
The  priority queue contains initially one element which represents the root node of the joint set enumeration tree (line \ref{line:algoIS_clustering_algo_insertCollRoot}) -- at the root, the traversal has to start. Then this queue is processed until it is empty.

If the first element of the queue is a \pattern, no better patterns can exist; in this case (and if the pattern is non-redundant to previously selected patterns), we can add it to the result set (line \ref{line:addClusterToResult}). In contrast, if the first element is a subtree, we continue with the best-first travel. 

In detail, 'continuing the traversal' means: Given the current subtree -- e.g. rooted at $O$ --, we first pick one of the vertices $u\in\bigcup_{i\in\{1,2\}} cand_{O,i}$ from $O$'s candidate set. Here, we use the vertex having the highest degree w.r.t. $O$ since it most probably leads to dense subgraphs.  We then check whether $O\cup \{u\}$ is a valid \pattern; if so, we add it to the queue. Besides the pattern, two further objects are added to the queue: the new subtree rooted at $O\cup \{u\}$, and  the subtree still rooted at $O$ -- but now removed by the branch which represents the already added $O\cup \{u\}$. For example, in Fig. \ref{fig:setenumtree}, assuming we are inspecting the subtree rooted at $\{v_1\}$, one would add the subtree rooted at $\{v_1,v_2\}$ as well as the subtree $\{v_1\}$ removed by the branch $\{v_1,v_2\}$ to the queue.

In general, by always adding these two subtrees to the queue, we ensure that each node in the joint set enumeration tree is exactly visited once. Simultaneously, the overall priority queue ensures that subtrees and patterns with a high interestingness are processed first: The queue is sorted by the (actual/estimated) interestingness values of the patterns/subtrees. Even more, since the estimated interestingness values of the subtrees are upper bounds on the actual interestingness of the contained patterns, this processing guarantees that patterns are added to the final result in monotonically decreasing order of their interestingness values -- thus, approximating our goal of Definition \ref{eq:problem}.

\algrenewcommand\algorithmicindent{0.8em}%

\algnewcommand{\LineComment}[1]{\State \(\triangleright\) #1}

\begin{algorithm}[t]
	
	\caption{\software: Best-first search for \patterns}
	\label{pseudocode} \small

	\begin{algorithmic}[1]
		
		\Require Graphs $G_1,G_2$ with $G_i=(V,E_i)$
		\Ensure Maximal set $Result$ of non-redundant patterns
		
		\State $Result:=\emptyset$ \label{line:resultempty}\Comment{Final result set}
		\State $queue:=[~(\emptyset,[cand_{\emptyset,i}=V]_{i \in \{1,2\}})~]$\label{line:algoIS_clustering_algo_insertCollRoot} \Comment{Priority queue of patterns \& subtrees}
		\While{$queue$ not empty}\label{line:algoIS_clustering_algo_QueueLoop}
		\State $Obj:=queue.\mathrm{pop}()$ \Comment{Select object with highest interestingness}
		\If{$Obj$ is a \pattern $O$} 
		\If{$\neg \exists P \in Result: O \red P \vee P \red O  $} $Result.add(O)$ \EndIf		\label{line:addClusterToResult} 	
		\Else \Comment{$Obj$ is subtree $ST=(O,[cand_{O,i}]_{i \in \{1,2\}})$}
		\State continue tree traversal at $ST$
		\LineComment{thus adding potential patterns and child-subtrees to 
		the queue}
		\EndIf
		\EndWhile
		\State \Return Result
\end{algorithmic}

\end{algorithm}

\begin{theorem}Algo.~\ref{pseudocode} generates a maximal set of non-redundant \patterns.
\end{theorem}
\begin{proof}(Sketch) (i) Clearly the result contains no redundancy as ensured by line \ref{line:addClusterToResult}. (ii) The tree traversal ensures that every node which might be a \pattern is visited. Further, every node that is a \pattern, will be added to the queue. Since the queue is processed until empty, every pattern is either redundant or added to the result. Thus, the result is maximal.$\hfill$
\end{proof}

\textbf{Computational Complexity: } \software follows an A*-like principle to enumerate contrasting patterns. Thus, its efficiency depends on the bounds for the interestingness of subtrees; we shall show below in Sec.~\ref{sec:quality_estimation} how we determine this upper bound. In the worst case, each node in the enumeration tree has to be visited leading to an exponential complexity in the number of vertices, as in the A*-search strategy. However, in practice our upper bounds combined with our pruning strategies result in fewer nodes visited in the enumeration tree. Indeed, in Sec.~\ref{exp:overall}, we study an alternative to \software by turning off all these techniques demonstrating the effectiveness of our pruning strategy in real-world datasets.

\subsection{Bounds for the Interestingness}
\label{sec:quality_estimation}
Last, we present our upper bound for the interestingness of subtrees. This bound is crucial for the best-first traversal and to find the most interesting patterns. Moreover, if the estimate is not positive we can even prune the whole subtree since no interesting patterns can be found in it.

Formally, we are interested in finding a function $I^*(O)$ such that
$$\text{  $\forall X\supset O,X\subseteq S_O$: $X$ is a \pattern}\Rightarrow I(X)\leq I^*(O) $$ where $S_O:=O\cup cand_{O,1}\cup cand_{O,2}$.
That is, the interestingness of each \pattern located in the subtree rooted at $O$ is bounded from above by $I^*(O)$. Of course, the bound needs to be efficiently computable, i.e we cannot simply enumerate all $X$. 

To derive the bound, let us first introduce a definition. The set of edges connecting two disjoint sets $A$ and $B$ in graph $i$ is denoted with 
$$E_i(A\leftrightarrow B):=\{(u,v)\in E_i \mid u\in A, v\in B\}$$
Thus, $|E_i(A\leftrightarrow B)|=\sum_{a\in A}deg_{G_i}^B(a)=\sum_{b\in B}deg_{G_i}^A(b)$ where $deg_{G_i}^Y(x)$ (see. Def. \ref{def:clique}) denotes the degree of vertex $x$ w.r.t. the set $Y$ in graph $i$.

Since the interestingness (Eq. \eqref{eq:inter}) decomposes into the difference in the number of edges and the cardinality of the set, we first derive the following result:

\begin{theorem}
Assume $X$, $O\subset X\subseteq S_O$, is a \pattern which forms a quasi-clique in graph $i$, it holds:
$|E_i(X)|-|E_j(X)|\leq d_{ij}(O)$ where

\begin{equation}
\begin{split}
d_{ij}(O):=|E_i(O)|-|E_j(O)|+\!\!\!\!\!\!\!\!\!\sum_{v\in cand_{O,i}}\!\!\!\!\max\{0,deg_{G_i}^O(v)- \\deg_{G_j}^O(v)+ \frac{1}{2}deg_{G_i}^{cand_{O,i}}(v)\}
\end{split}
\end{equation}

\end{theorem}

\begin{proof}
	Clearly $E_i(X)=E_i(O)\cup E_i(X\backslash O) \cup E_i(O\leftrightarrow X\backslash O)$. And since the above sets are disjoint 
	$|E_i(X)|=|E_i(O)|+| E_i(X\backslash O) |+ |E_i(O\leftrightarrow X\backslash O)|$.
	 Further, since we consider undirected graphs we have
	 $|E_i(X\backslash O)|= \frac{1}{2}\sum_{v\in X\backslash O}  deg_{G_i}^{X\backslash O}(v) $. And since $X$ is assumed to be a quasi-clique in graph $i$ (i.e. $X\subseteq O\cup cand_{O,i}$), we have $deg_{G_i}^{X\backslash O}(v)\leq deg_{G_i}^{cand_{O,i}}(v)$.
	 It follows, $|E_i(X)|-|E_j(X)|= |E_i(O)|-|E_j(O)|+ |E_i(O\leftrightarrow X\backslash O)|-|E_j(O\leftrightarrow X\backslash O)|+
	E_i(X\backslash O) |-E_j(X\backslash O) |
	\leq
	|E_i(O)|-|E_j(O)|+\sum_{v\in X\backslash O}[deg_{G_i}^O(v)-deg_{G_j}^O(v)+\frac{1}{2}deg_{G_i}^{X\backslash O}(v)]
	\leq	
	|E_i(O)|-|E_j(O)|+\sum_{v\in X\backslash O}\max\{0,deg_{G_i}^O(v)-deg_{G_j}^O(v)+\frac{1}{2} deg_{G_i}^{cand_{O,i}}(v)\}
	\leq
	|E_i(O)|-|E_j(O)|+\sum_{v\in cand_{O,i}}\max\{0,$ $deg_{G_i}^O(v)-deg_{G_j}^O(v)+\frac{1}{2} deg_{G_i}^{cand_{O,i}}(v)\} 
	$ $\hfill$
\end{proof}

The above result gives an estimate on the edge difference in a \pattern. Note that it is highly efficient to compute since we only have to iterate once through all vertices in the candidate set $cand_{O,i}$. The number of edges in $O$ as well as the degrees can easily be maintained during the run of the algorithm and don't need to be recomputed all the time.
Using the above result leads to:
\begin{theorem}
	Assume $X$, $O\subset X\subseteq S_O$, is a \pattern, it holds:
	$I(X)\leq I^*(O):=\frac{2}{|O|} \max\{d_{12}(O),d_{21}(O)\}$
\end{theorem}

\begin{proof}
$ I(X) = \frac{2\cdot \text{abs}(|E_1(X)|-|E_2(X)|)}{|X|-1}   
\leq$ $ \frac{2\cdot \text{abs}(|E_1(X)|-|E_2(X)|)}{|O|} 
\leq\frac{2\cdot \max\{d_{12}(O),d_{21}(O)\}}{|O|}
$ 
\end{proof}

Given that $I(X) \leq I^*(O)$, we can safely prune the whole subtree when $I^*(O) \leq 0$ since there is no contrast at all.

\section{Experimental Evaluation}
\label{sec:results}

In the following, we evaluate the quality of the detected patterns and runtime of our algorithm on synthetic and real-world datasets.\footnote{Our experiments were conducted on a Core i7 3.5 GHz CPU with Java8 64 bit.} As there is no ground-truth available regarding contrasting patterns, which hinders the evaluation of the algorithm, we highlight some characteristics of the resulting contrasting patterns. 

\textbf{Parameter selection.} The redundancy parameter $r$ along with $\delta$ and $\delta'$ play a relevant role in the task of mining contrasting patterns. To showcase our \pattern model and the \software algorithm, in our experimental setting, we have used $\delta \geq 0.5$ and $\delta' = 0$;~\ie~we are interested in any contrasting pattern that is tightly and evenly connected in at least one graph (see Sec.~\ref{sec:patternModel}). Also, we have set $r=0.1$ as our interest is in patterns with low redundancy.

Still, we believe that these parameters should be selected according to the application of interest. Thus, \textbf{our algorithm and all used datasets are available at URL}.\footnote{In order to keep this submission double-blinded the address will be available upon acceptance}

\begin{figure*}[t]%
	\centering
	\vspace*{-1mm}
	\includegraphics[clip, scale = 0.85]{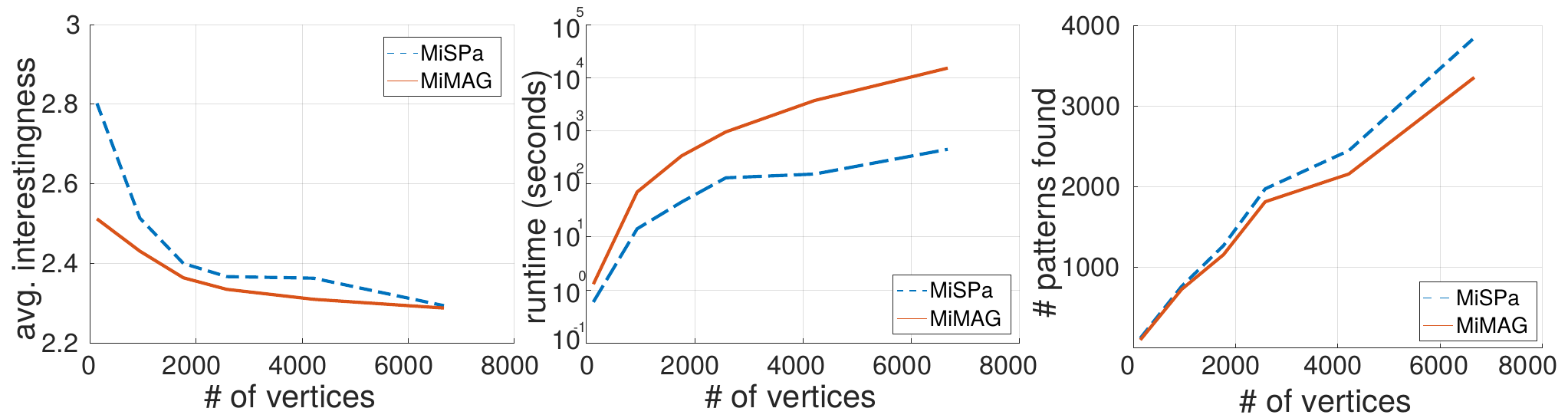} 
\vspace*{-2mm}	
\caption{Quality, runtime, and number of patterns of \software contrasted against MiMAG. }
	\label{fig:quality}
	
\end{figure*}

%

\textbf{Competing approaches.} Intuitively, to detect contrasting patterns, it would be possible to consider the complement graphs and apply a dense subgraph extraction algorithm on them. That is, given two graphs $G_1 = (V, E_1)$ and $G_2 = (V, E_2)$ build their complements $\overline{G_1} = (V, \overline{E_1})$, $\overline{G_2} = (V, \overline{E_2})$ where $\overline{E_i} = (V \times V) \setminus E_i$. Then, apply a cross-graph quasi-clique method first on the combination $G_1,\overline{G_2}$ and then on $\overline{G_1}, G_2 $. Next, calculate the interestingness of the patterns considering our interestingness function (\ie~in this step we analyze both graphs $G_1$ and $G_2$). Before reporting the final resulting set, we remove redundant patterns using our redundancy measure. 

Notice that it is crucial to run the method only on the combinations $\overline{G_i}, G_j $ $i \neq j$ since sparse patterns might appear in either of the four input graphs.\footnote{Clearly, patterns over the combination $\overline{G_1}, \overline{G_2}$ and $G_1, G_2$ are misleading. In the first case, patterns are irrelevant, in the second we are computing regular cross-graph quasi-cliques.}

For our experimental evaluation we have selected MiMAG~\cite{mimag} as the competing approach. The rationale is twofold: First, it is based on the $\delta$-Quasi-Clique definition\footnote{In MiMAG, however, is referred to as $\gamma$-Quasi-cliques.} as our \patterns allowing a more fair comparison, second, it allows fast and efficient computation of cross-graph $\delta$-quasi-cliques. Regarding MiMAG parameters we have set $\gamma = \delta$. In our experiments contrasting patterns detected with MiMAG on complement graphs are simply denoted MiMAG.

For further evaluation of our pruning and estimation techniques we study an alternative of \software by turning off all these techniques.

\subsection{Evaluation on synthetic graphs}
\label{sec:synthetic}

We have evaluated our approach by generating synthetic graphs containing contrasting patterns as well as vertices and edges that do not belong to any pattern (\ie~noise vertices and edges). For this, we generated two random graphs following a power-law distribution and we randomly embedded quasi-cliques of size 10 and density 0.6 to each of them. Since, due to the random embedding, the quasi-cliques cover different node sets in both graphs, we will observe contrast in the two graphs. We scaled the size of these graphs from 110 vertices (442 edges) to 6672 vertices (29464~edges). 

In Fig.~\ref{fig:quality} we report the quality and runtime of \software and MiMAG on these graphs. \software produces patterns with slightly higher interestingness (left plot). 
Clearly, mining contrasting patterns with \software and analyzing the complement graphs with MiMAG are highly related. The strong difference, though, becomes clear when considering the runtime.
As shown in the middle plot, \software clearly outperforms MiMAG (note the logarithmic scale). Indeed, as the size of the graph increases, so does the difference in the running time between both approaches. The approach based on complement graphs does not scale to larger datasets.

It is also important to note that the lower runtime of \software is not simply due to a smaller number of patterns (trivially, zero runtime could be achieved by reporting no patterns at all). Fig.~\ref{fig:quality} (right) shows that the number of patterns found by \software is even slightly larger than MiMAG;~\eg~in the largest synthetic graph \software detected 3848 patterns while MiMAG 3356. 

As shown in Fig.~\ref{fig:density}, MiMAG favors the detection of dense and large quasi-cliques suggesting a trade-off between the size-density and the number-quality of the detected patterns. \software favors the detection of high contrasting patterns while MiMAG large and dense patterns.

\begin{figure*}[t]%
	\centering
	\vspace*{-1mm}
	\includegraphics[clip, scale = 0.85]{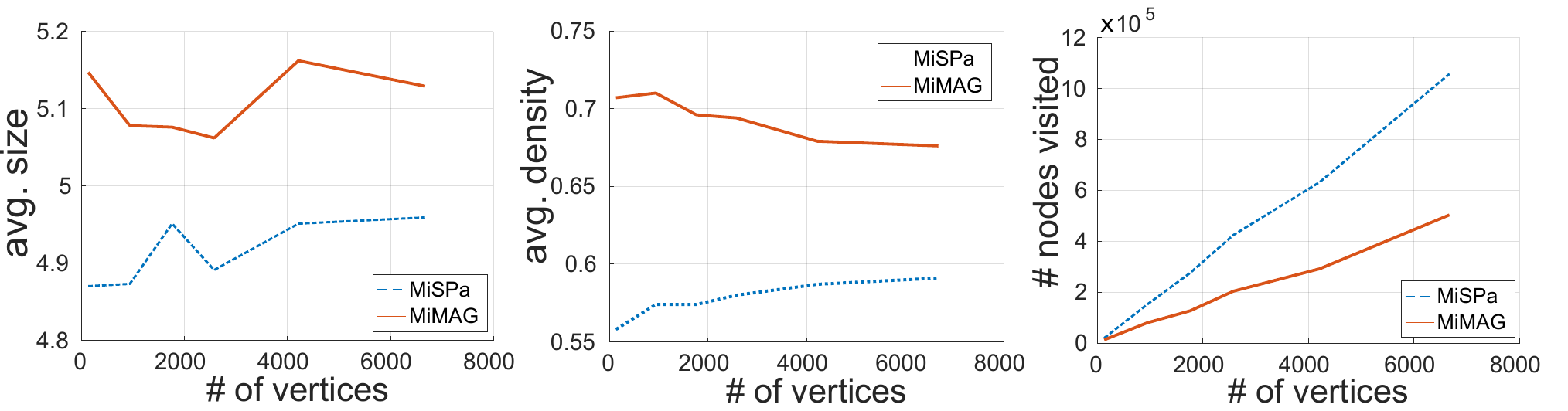} 
	\vspace*{-2mm}	
	\caption{Size, Density, and number of visited nodes in the enumeration tree of \software contrasted against MiMAG. }
	\label{fig:density}
	
\end{figure*}

Given that \software produces more patterns, it visits more nodes in the set enumeration tree (see right plot Fig.~\ref{fig:density}) but obtains fewer runtime. The underlying reason is that our used pruning techniques are more effective and efficient in sparse graphs. For example, some of the techniques discard unpromising nodes based on the graphs' diameter \cite{pjz05}; these principles are more expensive to compute in the dense complement graphs used in MiMAG and do not provide a significant pruning.

%
%
%


\subsection{Evaluation on real-world datasets: Case Studies}\label{exp:case}

Our next evaluation considers four real-world datasets. We provide  case studies (Sec. \ref{exp:case}) as well as overall statistics and performance results (Sec. \ref{exp:overall}).


\textbf{Arxiv.} The first real-world dataset considers a portion of the Arxiv database\footnote{\url{https://www.cs.cornell.edu/projects/kddcup/datasets.html}}. Here, vertices of the graph  are papers. 
For the first graph, two papers \texttt{A} and \texttt{B} are connected by an edge, if \texttt{A} cites or is being
 referenced by \texttt{B}. A second graph is constructed, considering the same papers, and computing  the similarity of the abstracts. We used word2vec embeddings and cosine distance to measure similarity. Two papers are connected by an edge, if their similarity is larger than 0.98. In principle, papers with similar topics should be highly connected in the citation/reference graph. 

 So, our interest is to find similar papers without common cites, or cited papers in dissimilar topics. Overall, there are 6156 vertices and 20515 edges. 
 
 	 \begin{figure}[h]
 	 	\centering	
 	 	\includegraphics[scale = 0.45,trim= 0cm 0cm 0 0] {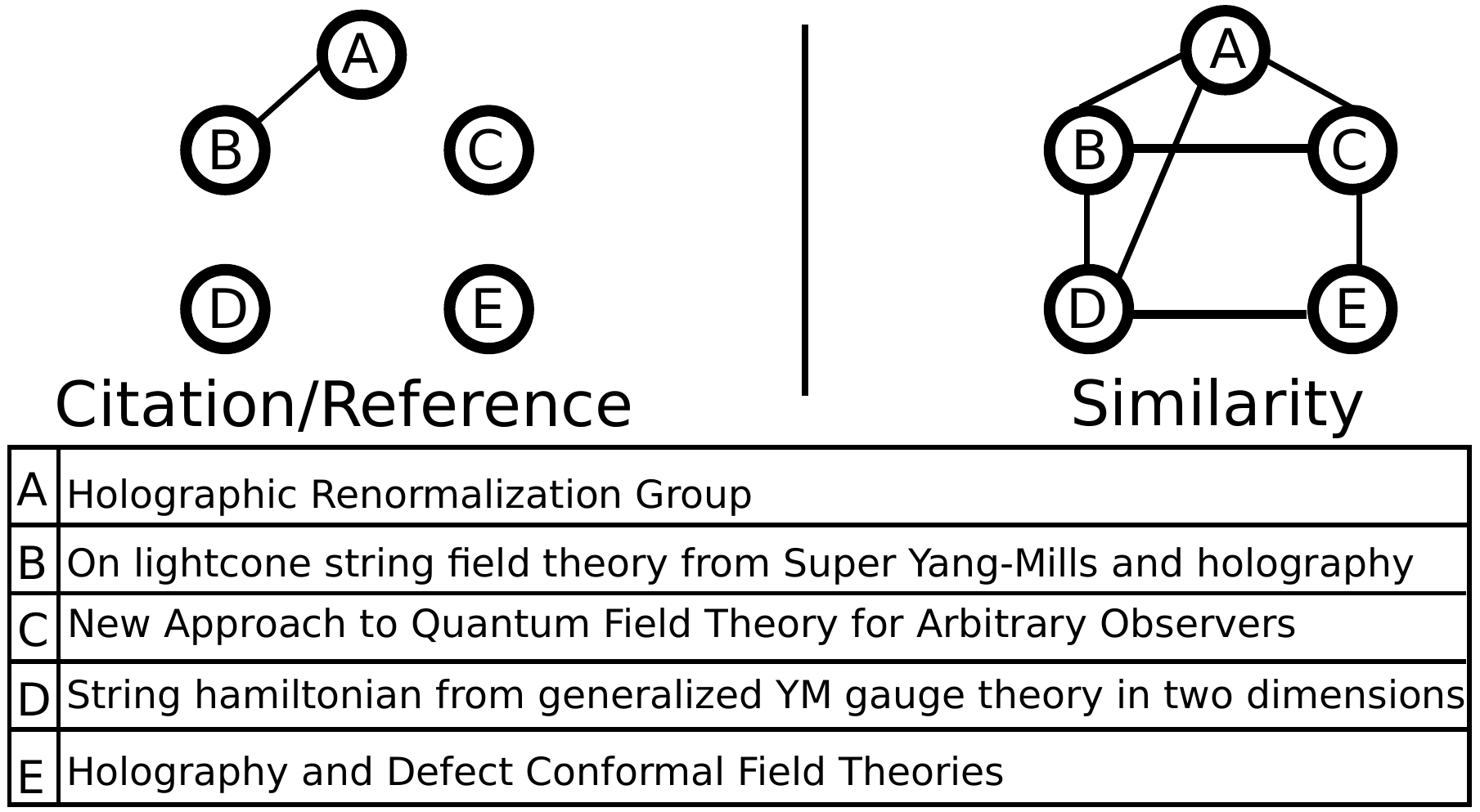}
 	 	\vspace*{-1mm}
 	 	\caption{Papers with textual similarity and few citations/references.}
 	 	\label{fig:arxiv}
 	 \end{figure}
 
Fig.~\ref{fig:arxiv} illustrates one contrasting pattern. Note five papers forming a 0.5-quasi-clique in the Similarity graph but with only few citations or references. Particularly, papers \texttt{A}, \texttt{B}, and \texttt{D} are highly related (forming a clique) to the Yang-Mill theory while the rest of them are partially similar and arise in the context of string theory. Despite the similarity, only papers \texttt{A} and \texttt{B} have a connection in the citation graph. The interestingness of the pattern is~3.

\textbf{Friendship/Facebook.} In the following experiment we have contrasted real-life friendship versus Facebook friendship among High School students, considering the data  in~\cite{Mastrandrea16}. 
In this graph we have over 208 vertices and 1843 edges.
\begin{figure}[h]
	\vspace*{-1mm}\centering
	\includegraphics[scale = 0.4] {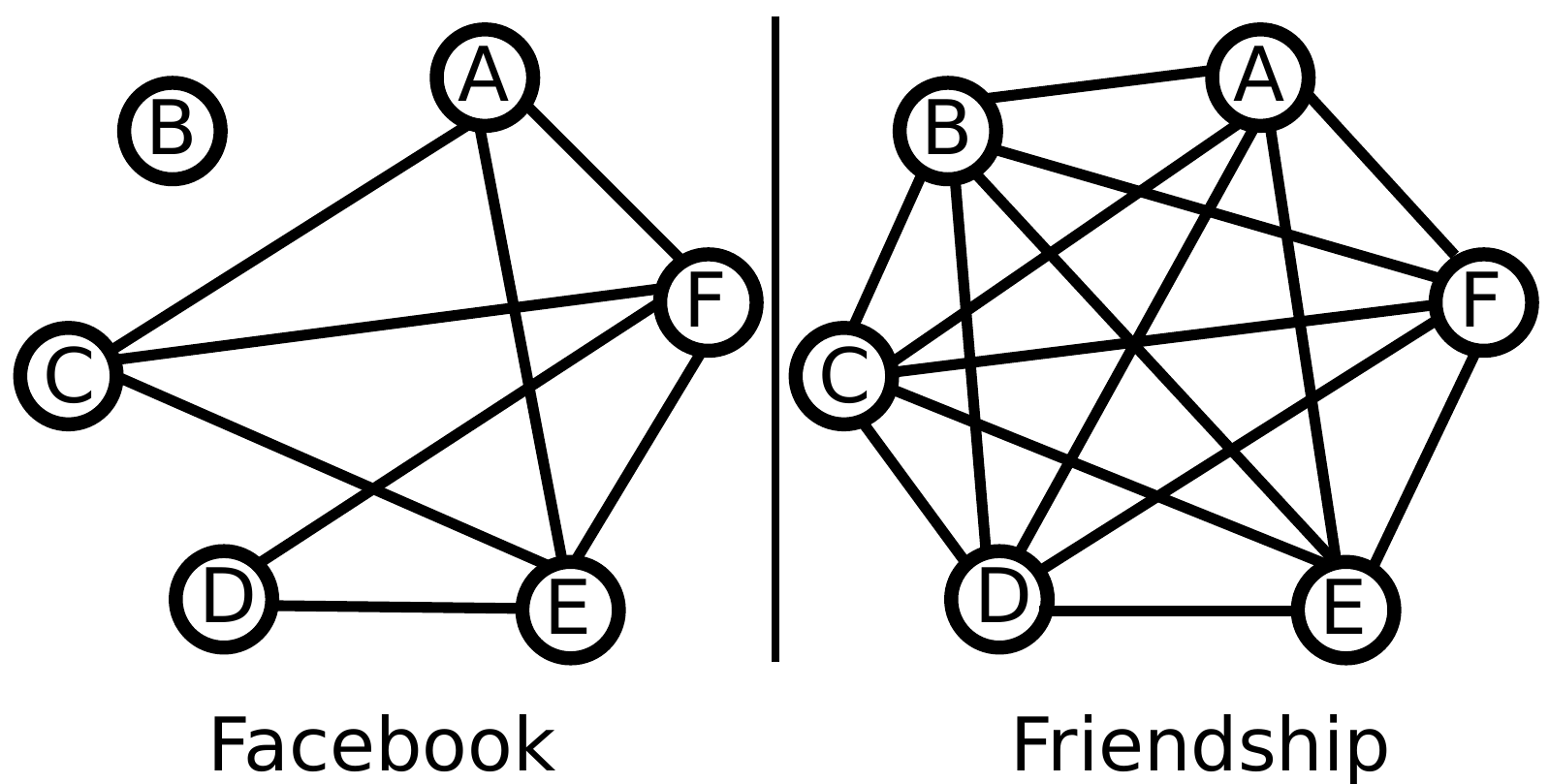}
	\vspace*{-1mm}	\caption{A \pattern w.r.t. friendship.}
	\label{fig:facebook_vs_friends}
\end{figure}

Fig.~\ref{fig:facebook_vs_friends} exhibits a contrasting pattern where all six students have a real-life friendship, but they are very sparse in the Facebook graph. Interestingly, student \texttt{B} is not a Facebook friend of the other five students despite being a highly connected in the real-life graph. The interestingness of this pattern is 2.8.


\textbf{DBLP.} In our next case study, we have used a portion of the DBLP database\footnote{\url{http://dblp.uni-trier.de}}. In this experiment, each author is a vertex of the KDD, ICDM, VLDB, ICML, WWW, and PAKDD venues, and there is an edge connecting two authors if they have co-authored at least two papers between the years 2000 and 2007. Likewise, a second graph is constructed considering co-authorship between 2008 and 2015. Intuitively, we aim at finding collaborations that start or end over time. This graph contains 5319 vertices and 7012 edges.
%

%
%
 

In Fig.~\ref{fig:dblp2} we can see an example of a collaboration that goes from sparse to highly dense. For example, Gupta and Gosh are no longer co-authors after 2007, in our co-authorship graph, and Gupta became connected with other Ghosh co-authors. By contrast, Ghosh remains collaborating with Dhillon and Deodhar. The interestingness of the pattern is 3.
\begin{figure}[t]
	\vspace*{-1mm}
	\centering
	\includegraphics[scale = 0.5, trim = 0 0 0 0cm] {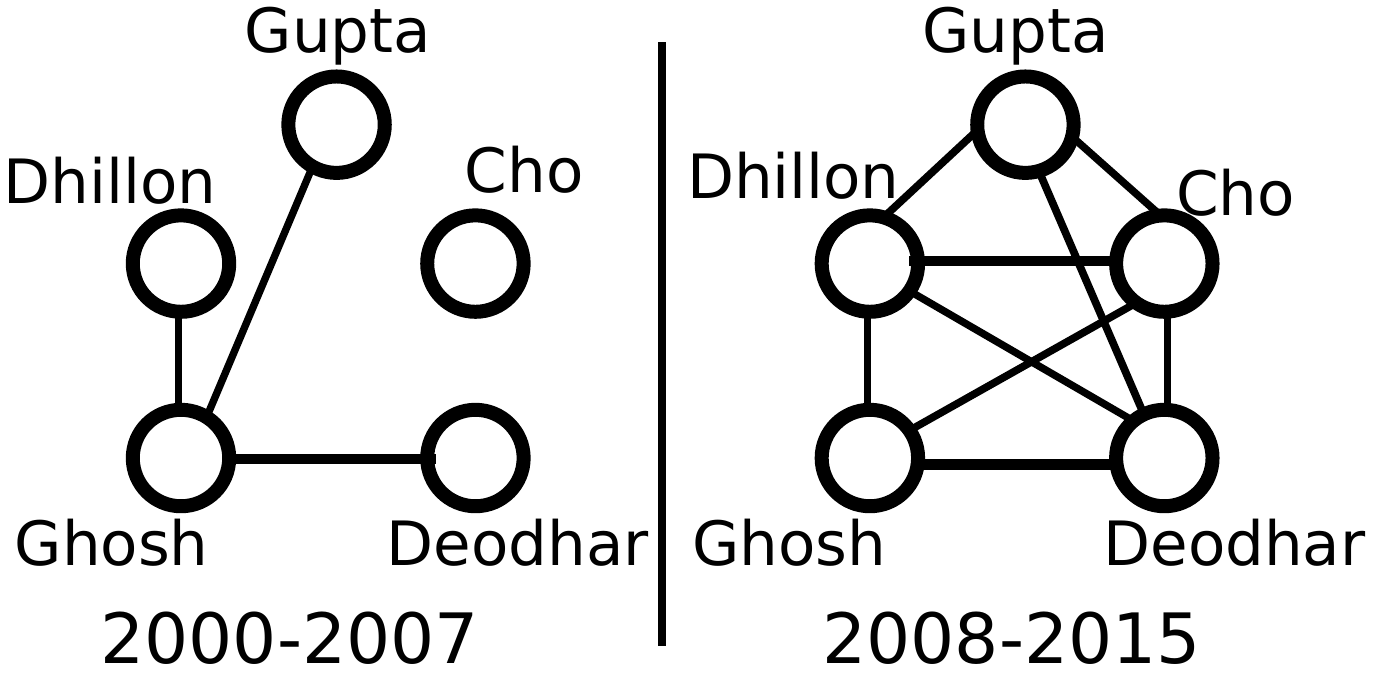}
	\vspace*{-1mm}\caption{A contrasting pattern in the co-authorship graphs.}
	\label{fig:dblp2}
\end{figure}
\newcommand*\rot{\rotatebox{90}}
\begin{table*}[t]
	
	\begin{tabular}{|l|l|l|l|l|l|l|l|l|}
		\hline
		& \multicolumn{2}{c|}{DBLP}    & \multicolumn{2}{c|}{Arxiv}          & \multicolumn{2}{c|}{Laws}           & \multicolumn{2}{c|}{Friendship}      \\ \hline
		& \rot{\software} & \rot{MiMAG}   & \rot{\software} & \rot{MiMAG}   & \rot{\software} & \rot{MiMAG}   & \rot{\software} & \rot{MiMAG}    \\ \hline
		runtime (min)                             & \textbf{0.08}   & 313.5           & \textbf{0.15}   & 584.5           & \textbf{0.009}  & 2.1             & \textbf{2.52}   & \textbf{2.52}  \\ \hline
		nodes visited $\times10^4$ & 11.58             & \textbf{2.41} & \textbf{1.75}   & 2.42            & 0.95              & \textbf{0.58} & 240.6             & \textbf{41.36} \\ \hline
		contrast: avg($I(O)$)                     & \textbf{3.72}   & 3.62            & 4.02              & \textbf{4.05} & \textbf{5.79}   & 5.77            & 4.7               & \textbf{6.2}   \\ \hline
		contrast: sum($I(O)$)                     & \textbf{1255.7} & 698.6           & \textbf{1950}   & 1200            & \textbf{81}     & 75              & \textbf{539.7}  & 254.4            \\ \hline
		density: avg($\gamma$)                    & 0.57              & \textbf{0.74} & \textbf{1}      & \textbf{1}    & 0.78              & \textbf{0.85} & 0.86              & \textbf{0.88}  \\ \hline
		size: avg(|O|)                            & 5.32              & \textbf{5.38} & \textbf{4.04}   & \textbf{4.04} & 5.12              & \textbf{6.84} & 5.55              & \textbf{7.22}  \\ \hline
	\end{tabular}
	\vspace*{0mm}
	\caption{Performance on real data. \software detects high contrast patterns with a  significantly lower runtime.}
	\label{tbl:results}	
	\vspace*{-2mm}
\end{table*}

\textbf{Law data.}
Our last experiment comes from a collection of civil laws in Germany (the BGB). To construct this graph we have considered a law (\ie~a paragraph) as a vertex in the graph. In the first graph there is an edge connecting two laws, if one mentions (or is being mentioned by) the other. Then, we have extracted the top 21 keywords for each paragraph, and constructed a second graph by connecting two laws if they share at least 8 keywords. In principle, laws highly connected in the citation (reference) graph should also be textually related. Hence, the goal is to find dissimilar laws highly connected in the citation graph, and similar laws with fewer connections in the citation graph. 
Overall, the graph contains 2402 vertices and 3054 edges. 

In Fig.~\ref{fig:laws} we illustrate one of these patterns. Notice that in terms of similarity the five laws form a 0.5-quasi-clique. This is expected since all these laws refer to the legal name of a child, so words like 'child' or 'parent' are very common.

	\begin{figure}[h]
	\vspace*{-1.5mm}
	\centering
	\includegraphics[scale = 0.49] {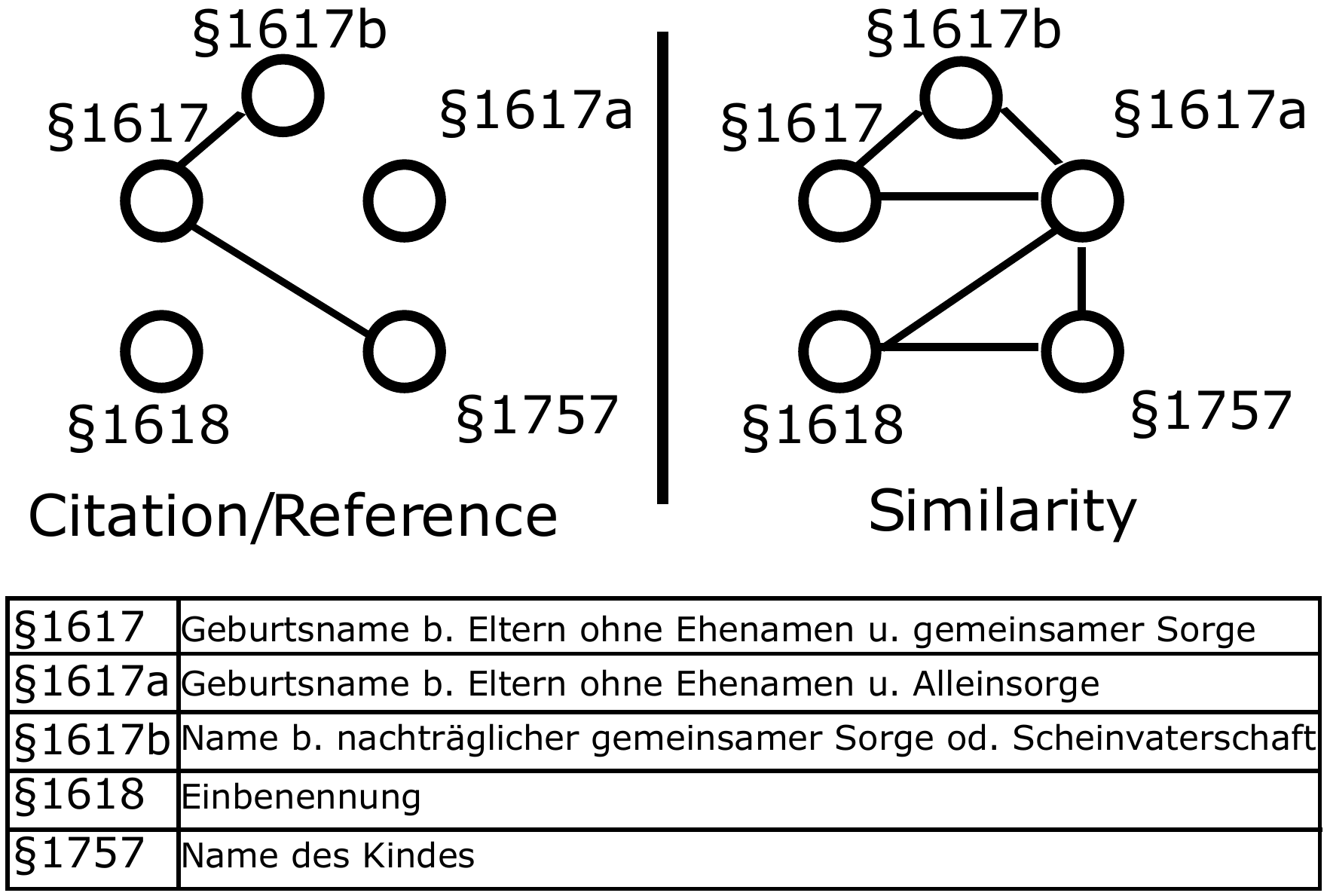}
	\vspace*{-1mm}
	\caption{A collection of laws with few citations/references, but having high textual similarity.}
	\label{fig:laws}
\end{figure}

Still, notice that $\S~1757$ and $\S~1617$ are textually dissimilar despite being connected in the Citation graph. This is because $\S~1757$ appears in the context of adoption while $\S~1617$ refers to the name of the child when one of the parents is not married to the biological parent of the child;~\ie~the citation is expected. The interestingness of this pattern is 2.
\subsection{Evaluation on real-world datasets: Overall performance}\label{exp:overall}

\textbf{Running times and quality of the patterns.} 
Table~\ref{tbl:results} summarizes the performance of our \software algorithm and contrasts it against the complement graph approach with MiMAG as the quasi-clique detector.

Regarding runtime, our approach outperforms MiMAG (except in the small friendship graph). As discussed in Section~\ref{sec:synthetic}, this result is related to the poor efficiency of MiMAG's pruning techniques on dense (complement) graphs. Note that, with the exception of the Arxiv experiment, \software visits more nodes in the enumeration tree which is expected since it considers more subtrees as potential \patterns and produces more of them.

Indeed, \software generates often patterns with slightly higher interestingness and -- despite its lower runtime -- is able to find more interesting patterns, which is clearly illustrated in the sum of interestingness of all the patterns found, sum($I(O)$). In some cases (\eg~in the DBLP experiment) this amounts almost twice the number of patterns reported by MiMAG. In contrast to \software, MiMAG is steered towards slightly more dense and large patterns since in almost all cases \software reported patterns with less density and smaller.

\subsection{Effectiveness of pruning} 

To test the effectiveness of our pruning and bounding techniques we have contrasted \software with and without these techniques. In Fig.~\ref{fig:runtime_prunVSnoprun} we report the runtime and number of nodes visited in the set enumeration tree. Due to the exponential complexity of \software without pruning strategies (see Sec.~\ref{sec:overallScheme}), we had to consider a smaller portion of the Arxiv, and the Facebook/Friendship datasets; the graphs have 4575 vertices/5837 edges, and 208 vertices/991 edges. To better illustrate the results we have plotted them using log~scale. 
\begin{figure}[h]
	\centering
	\includegraphics[clip, scale = 0.59] {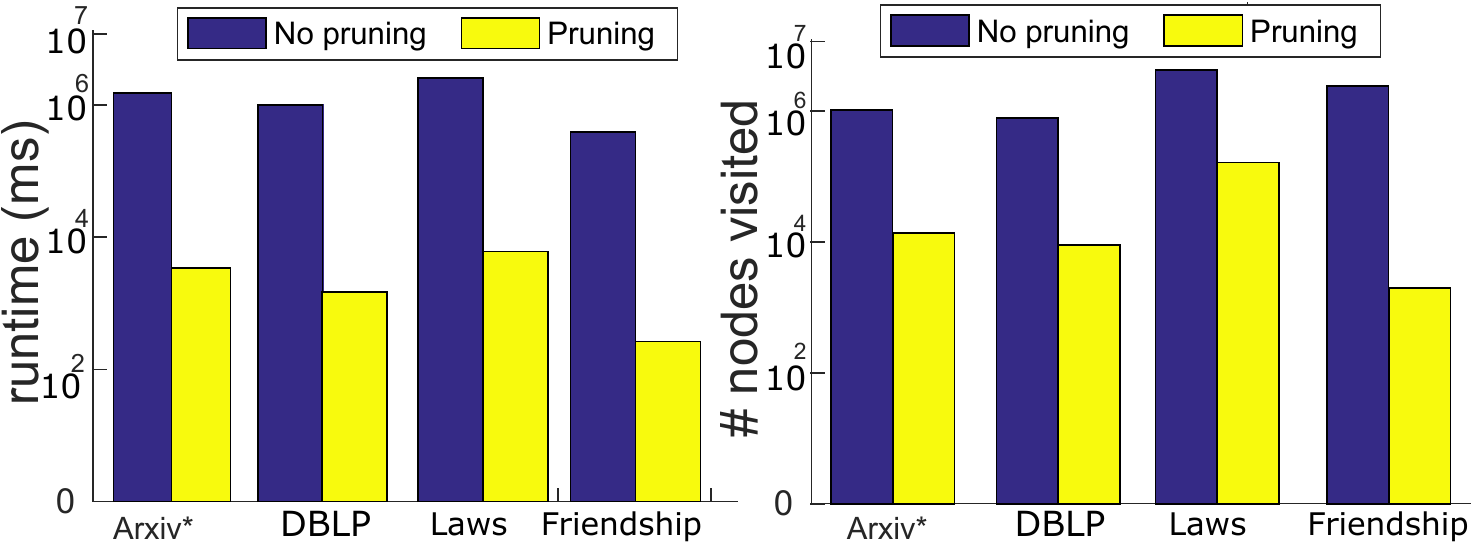}
	\vspace*{-1mm}
	\caption{Runtime and number of visits to nodes in the enumeration tree of \software with and without pruning techniques.}
	\vspace*{-1mm}
	\label{fig:runtime_prunVSnoprun}
\end{figure}

As shown in Fig.~\ref{fig:runtime_prunVSnoprun} (left) the runtime is dramatically affected without pruning. The Arxiv experiment took 443 times more time, the Conferences 676 times, the Law data experiment 420 times, and the Friendship experiment 1498 times. This can be explained by the  number of visits made to the enumeration tree as shown in Fig.~\ref{fig:runtime_prunVSnoprun} (right). For example, without pruning, the Arxiv experiment did 74 times more visits and the Conferences experiment made 87 times more visits. Interestingly, the law data produces the highest number of visits and runtime despite being the second smallest in terms of number of edges and vertices (\eg~to find all contrasting patterns, it takes 2567 seconds without pruning techniques and only 6 seconds with pruning). 
These results show that the pruning techniques highly improve the performance of \software.


Overall, as shown in these experiments, \software successfully detects contrasting quasi-clique patterns in a variety of datasets. The various examples indicate the usefulness of our novel pattern model and the potential of \software to efficiently find these patterns.

\section{Conclusion}
\label{sec:conclusions}
We proposed the new principle of {contrasting quasi-clique patterns}: subgraphs which are dense in one graph but sparse in another. Unlike existing works focusing on density only, our patterns highlight the difference/contrast between the subgraphs -- thus, opening the door for a novel way of knowledge extraction. We introduced a model aiming to find a set of non-redundant, most interesting \patterns. Based on this model, we proposed an efficient A$^*$-like algorithm using optimistic subtree estimates and pruning techniques. Using a variety of different real-world datasets we illustrated the efficiency of our method and how \patterns can be used to derive interesting insights from graphs.

While this paper introduced the first approach for finding \patterns, as future work we aim to derive even more scalable approaches, \eg, based on sampling, we plan to propose extensions to more complex multi-layer graphs, and to incorporate edge weights in the model.


%
%
\balance
\def\bibfont{\small}
\bibliographystyle{ACM-Reference-Format}
\bibliography{referenceQC}


\begin{thebibliography}{22}


\ifx \showCODEN    \undefined \def \showCODEN     #1{\unskip}     \fi
\ifx \showDOI      \undefined \def \showDOI       #1{#1}\fi
\ifx \showISBNx    \undefined \def \showISBNx     #1{\unskip}     \fi
\ifx \showISBNxiii \undefined \def \showISBNxiii  #1{\unskip}     \fi
\ifx \showISSN     \undefined \def \showISSN      #1{\unskip}     \fi
\ifx \showLCCN     \undefined \def \showLCCN      #1{\unskip}     \fi
\ifx \shownote     \undefined \def \shownote      #1{#1}          \fi
\ifx \showarticletitle \undefined \def \showarticletitle #1{#1}   \fi
\ifx \showURL      \undefined \def \showURL       {\relax}        \fi
\providecommand\bibfield[2]{#2}
\providecommand\bibinfo[2]{#2}
\providecommand\natexlab[1]{#1}
\providecommand\showeprint[2][]{arXiv:#2}

\bibitem[\protect\citeauthoryear{Aggarwal and Wang}{Aggarwal and Wang}{2010}]%
        {aggarwal2010managing}
\bibfield{author}{\bibinfo{person}{C.C. Aggarwal} {and} \bibinfo{person}{H.
  Wang}.} \bibinfo{year}{2010}\natexlab{}.
\newblock \bibinfo{booktitle}{\emph{Managing and Mining Graph Data}}.
  \bibinfo{series}{Advances in Database Systems}, Vol.~\bibinfo{volume}{40}.
\newblock \bibinfo{publisher}{Springer, New York}.
\newblock


\bibitem[\protect\citeauthoryear{Anchuri and Magdon-Ismail}{Anchuri and
  Magdon-Ismail}{2012}]%
        {anchuri2012communities}
\bibfield{author}{\bibinfo{person}{Pranay Anchuri} {and} \bibinfo{person}{Malik
  Magdon-Ismail}.} \bibinfo{year}{2012}\natexlab{}.
\newblock \showarticletitle{Communities and balance in signed networks: A
  spectral approach}. In \bibinfo{booktitle}{\emph{{ASONAM}}}.
  \bibinfo{pages}{235--242}.
\newblock


\bibitem[\protect\citeauthoryear{Boden, G{\"{u}}nnemann, Hoffmann, and
  Seidl}{Boden et~al\mbox{.}}{2012}]%
        {mimag}
\bibfield{author}{\bibinfo{person}{Brigitte Boden}, \bibinfo{person}{Stephan
  G{\"{u}}nnemann}, \bibinfo{person}{Holger Hoffmann}, {and}
  \bibinfo{person}{Thomas Seidl}.} \bibinfo{year}{2012}\natexlab{}.
\newblock \showarticletitle{Mining coherent subgraphs in multi-layer graphs
  with edge labels}. In \bibinfo{booktitle}{\emph{{SIGKDD}}}.
  \bibinfo{pages}{1258--1266}.
\newblock


\bibitem[\protect\citeauthoryear{Cai, Shao, He, Yan, and Han}{Cai
  et~al\mbox{.}}{2005}]%
        {cai2005community}
\bibfield{author}{\bibinfo{person}{Deng Cai}, \bibinfo{person}{Zheng Shao},
  \bibinfo{person}{Xiaofei He}, \bibinfo{person}{Xifeng Yan}, {and}
  \bibinfo{person}{Jiawei Han}.} \bibinfo{year}{2005}\natexlab{}.
\newblock \showarticletitle{Community Mining from Multi-relational Networks}.
  In \bibinfo{booktitle}{\emph{ECML PKDD}}. \bibinfo{pages}{445--452}.
\newblock


\bibitem[\protect\citeauthoryear{Cerf, Besson, Robardet, and Boulicaut}{Cerf
  et~al\mbox{.}}{2009a}]%
        {DBLP:journals/tkdd/CerfBRB09}
\bibfield{author}{\bibinfo{person}{Lo\"{\i}c Cerf},
  \bibinfo{person}{J{\'e}r{\'e}my Besson}, \bibinfo{person}{C{\'e}line
  Robardet}, {and} \bibinfo{person}{Jean-Fran\c{c}ois Boulicaut}.}
  \bibinfo{year}{2009}\natexlab{a}.
\newblock \showarticletitle{Closed patterns meet {\it n}-ary relations}.
\newblock \bibinfo{journal}{\emph{TKDD}} \bibinfo{volume}{3},
  \bibinfo{number}{1} (\bibinfo{year}{2009}).
\newblock


\bibitem[\protect\citeauthoryear{Cerf, Nguyen, and Boulicaut}{Cerf
  et~al\mbox{.}}{2009b}]%
        {DBLP:conf/ismis/CerfNB09}
\bibfield{author}{\bibinfo{person}{Lo\"{\i}c Cerf}, \bibinfo{person}{Tran
  Bao~Nhan Nguyen}, {and} \bibinfo{person}{Jean-Fran\c{c}ois Boulicaut}.}
  \bibinfo{year}{2009}\natexlab{b}.
\newblock \showarticletitle{Discovering Relevant Cross-Graph Cliques in Dynamic
  Networks}. In \bibinfo{booktitle}{\emph{{ISMIS}}}. \bibinfo{pages}{513--522}.
\newblock


\bibitem[\protect\citeauthoryear{Esmailian and Jalili}{Esmailian and
  Jalili}{2015}]%
        {esmailian2015community}
\bibfield{author}{\bibinfo{person}{Pouya Esmailian} {and}
  \bibinfo{person}{Mahdi Jalili}.} \bibinfo{year}{2015}\natexlab{}.
\newblock \showarticletitle{Community detection in signed networks: The role of
  negative ties in different scales}.
\newblock \bibinfo{journal}{\emph{Scientific reports}} \bibinfo{volume}{5},
  \bibinfo{number}{14339} (\bibinfo{date}{September} \bibinfo{year}{2015}).
\newblock


\bibitem[\protect\citeauthoryear{G\"unnemann, F\"arber, Boden, and
  Seidl}{G\"unnemann et~al\mbox{.}}{2010}]%
        {GFB+10}
\bibfield{author}{\bibinfo{person}{Stephan G\"unnemann}, \bibinfo{person}{Ines
  F\"arber}, \bibinfo{person}{Brigitte Boden}, {and} \bibinfo{person}{Thomas
  Seidl}.} \bibinfo{year}{2010}\natexlab{}.
\newblock \showarticletitle{Subspace Clustering Meets Dense Subgraph Mining: A
  Synthesis of Two Paradigms}. In \bibinfo{booktitle}{\emph{{ICDM}}}.
  \bibinfo{pages}{845--850}.
\newblock


\bibitem[\protect\citeauthoryear{Jin and Wang}{Jin and Wang}{2011}]%
        {jin2011lts}
\bibfield{author}{\bibinfo{person}{Ning Jin} {and} \bibinfo{person}{Wei Wang}.}
  \bibinfo{year}{2011}\natexlab{}.
\newblock \showarticletitle{LTS: Discriminative subgraph mining by learning
  from search history}. In \bibinfo{booktitle}{\emph{{ICDE}}}.
  \bibinfo{pages}{207--218}.
\newblock


\bibitem[\protect\citeauthoryear{Kim and Lee}{Kim and Lee}{2015}]%
        {Kim15}
\bibfield{author}{\bibinfo{person}{Jungeun Kim} {and} \bibinfo{person}{Jae-Gil
  Lee}.} \bibinfo{year}{2015}\natexlab{}.
\newblock \showarticletitle{Community Detection in Multi-Layer Graphs: A
  Survey}.
\newblock \bibinfo{journal}{\emph{SIGMOD Rec.}} \bibinfo{volume}{44},
  \bibinfo{number}{3} (\bibinfo{date}{December} \bibinfo{year}{2015}),
  \bibinfo{pages}{37--48}.
\newblock
\showISSN{0163-5808}


\bibitem[\protect\citeauthoryear{Liu and Wong}{Liu and Wong}{2008}]%
        {DBLP:conf/pkdd/LiuW08}
\bibfield{author}{\bibinfo{person}{Guimei Liu} {and} \bibinfo{person}{Limsoon
  Wong}.} \bibinfo{year}{2008}\natexlab{}.
\newblock \showarticletitle{Effective Pruning Techniques for Mining
  Quasi-Cliques}. In \bibinfo{booktitle}{\emph{{ECML PKDD}}}.
  \bibinfo{pages}{33--49}.
\newblock


\bibitem[\protect\citeauthoryear{Mastrandrea, Fournet, and Barrat}{Mastrandrea
  et~al\mbox{.}}{2015}]%
        {Mastrandrea16}
\bibfield{author}{\bibinfo{person}{Rossana Mastrandrea}, \bibinfo{person}{Julie
  Fournet}, {and} \bibinfo{person}{Alain Barrat}.}
  \bibinfo{year}{2015}\natexlab{}.
\newblock \showarticletitle{Contact Patterns in a High School: A Comparison
  between Data Collected Using Wearable Sensors, Contact Diaries and Friendship
  Surveys}.
\newblock \bibinfo{journal}{\emph{PLOS ONE}} \bibinfo{volume}{10},
  \bibinfo{number}{9} (\bibinfo{date}{09} \bibinfo{year}{2015}),
  \bibinfo{pages}{1--26}.
\newblock


\bibitem[\protect\citeauthoryear{McAuley, Pandey, and Leskovec}{McAuley
  et~al\mbox{.}}{2015}]%
        {McAuley15}
\bibfield{author}{\bibinfo{person}{Julian McAuley}, \bibinfo{person}{Rahul
  Pandey}, {and} \bibinfo{person}{Jure Leskovec}.}
  \bibinfo{year}{2015}\natexlab{}.
\newblock \showarticletitle{Inferring Networks of Substitutable and
  Complementary Products}. In \bibinfo{booktitle}{\emph{{SIGKDD}}}.
  \bibinfo{pages}{785--794}.
\newblock


\bibitem[\protect\citeauthoryear{Moser, Colak, Rafiey, and Ester}{Moser
  et~al\mbox{.}}{2009}]%
        {copam}
\bibfield{author}{\bibinfo{person}{Flavia Moser}, \bibinfo{person}{Recep
  Colak}, \bibinfo{person}{Arash Rafiey}, {and} \bibinfo{person}{Martin
  Ester}.} \bibinfo{year}{2009}\natexlab{}.
\newblock \showarticletitle{Mining Cohesive Patterns from Graphs with Feature
  Vectors}. In \bibinfo{booktitle}{\emph{{SDM}}}. \bibinfo{pages}{593--604}.
\newblock


\bibitem[\protect\citeauthoryear{M{\"u}ller, Assent, G{\"u}nnemann, Krieger,
  and Seidl}{M{\"u}ller et~al\mbox{.}}{2009}]%
        {rescu}
\bibfield{author}{\bibinfo{person}{E. M{\"u}ller}, \bibinfo{person}{I. Assent},
  \bibinfo{person}{S. G{\"u}nnemann}, \bibinfo{person}{R. Krieger}, {and}
  \bibinfo{person}{T. Seidl}.} \bibinfo{year}{2009}\natexlab{}.
\newblock \showarticletitle{Relevant subspace clustering: Mining the most
  interesting non-redundant concepts in high dimensional data}. In
  \bibinfo{booktitle}{\emph{{ICDM}}}. \bibinfo{pages}{377--386}.
\newblock


\bibitem[\protect\citeauthoryear{Pei, Jiang, and Zhang}{Pei
  et~al\mbox{.}}{2005}]%
        {pjz05}
\bibfield{author}{\bibinfo{person}{J. Pei}, \bibinfo{person}{D. Jiang}, {and}
  \bibinfo{person}{A. Zhang}.} \bibinfo{year}{2005}\natexlab{}.
\newblock \showarticletitle{{On mining cross-graph quasi-cliques}}. In
  \bibinfo{booktitle}{\emph{{SIGKDD}}}. \bibinfo{pages}{228--238}.
\newblock


\bibitem[\protect\citeauthoryear{Rymon}{Rymon}{1992}]%
        {DBLP:conf/kr/Rymon92}
\bibfield{author}{\bibinfo{person}{Ron Rymon}.}
  \bibinfo{year}{1992}\natexlab{}.
\newblock \showarticletitle{Search through Systematic Set Enumeration}. In
  \bibinfo{booktitle}{\emph{{KR}}}. \bibinfo{pages}{539--550}.
\newblock


\bibitem[\protect\citeauthoryear{Thoma, Cheng, Gretton, Han, Kriegel, Smola,
  Song, Yu, Yan, and Borgwardt}{Thoma et~al\mbox{.}}{2010}]%
        {thoma2010discriminative}
\bibfield{author}{\bibinfo{person}{Marisa Thoma}, \bibinfo{person}{Hong Cheng},
  \bibinfo{person}{Arthur Gretton}, \bibinfo{person}{Jiawei Han},
  \bibinfo{person}{Hans-Peter Kriegel}, \bibinfo{person}{Alex Smola},
  \bibinfo{person}{Le Song}, \bibinfo{person}{Philip~S Yu},
  \bibinfo{person}{Xifeng Yan}, {and} \bibinfo{person}{Karsten~M Borgwardt}.}
  \bibinfo{year}{2010}\natexlab{}.
\newblock \showarticletitle{Discriminative frequent subgraph mining with
  optimality guarantees}.
\newblock \bibinfo{journal}{\emph{Statistical Analysis and Data Mining}}
  \bibinfo{volume}{3}, \bibinfo{number}{5} (\bibinfo{year}{2010}),
  \bibinfo{pages}{302--318}.
\newblock


\bibitem[\protect\citeauthoryear{Ting and Bailey}{Ting and Bailey}{2006}]%
        {ting2006mining}
\bibfield{author}{\bibinfo{person}{Roger Ming~Hieng Ting} {and}
  \bibinfo{person}{James Bailey}.} \bibinfo{year}{2006}\natexlab{}.
\newblock \showarticletitle{Mining minimal contrast subgraph patterns}. In
  \bibinfo{booktitle}{\emph{{SDM}}}. \bibinfo{pages}{639--643}.
\newblock


\bibitem[\protect\citeauthoryear{Uno}{Uno}{2010}]%
        {Uno2010}
\bibfield{author}{\bibinfo{person}{Takeaki Uno}.}
  \bibinfo{year}{2010}\natexlab{}.
\newblock \showarticletitle{An Efficient Algorithm for Solving Pseudo Clique
  Enumeration Problem}.
\newblock \bibinfo{journal}{\emph{Algorithmica}} \bibinfo{volume}{56},
  \bibinfo{number}{1} (\bibinfo{date}{01 Jan} \bibinfo{year}{2010}),
  \bibinfo{pages}{3--16}.
\newblock


\bibitem[\protect\citeauthoryear{Wang, Zeng, and Zhou}{Wang
  et~al\mbox{.}}{2006}]%
        {crossgraph3}
\bibfield{author}{\bibinfo{person}{Jianyong Wang}, \bibinfo{person}{Zhiping
  Zeng}, {and} \bibinfo{person}{Lizhu Zhou}.} \bibinfo{year}{2006}\natexlab{}.
\newblock \showarticletitle{CLAN: An Algorithm for Mining Closed Cliques from
  Large Dense Graph Databases}. In \bibinfo{booktitle}{\emph{{ICDE}}}.
  \bibinfo{pages}{73--83}.
\newblock


\bibitem[\protect\citeauthoryear{Zeng, Wang, Zhou, and Karypis}{Zeng
  et~al\mbox{.}}{2006}]%
        {zwzk06}
\bibfield{author}{\bibinfo{person}{Z. Zeng}, \bibinfo{person}{J. Wang},
  \bibinfo{person}{L. Zhou}, {and} \bibinfo{person}{G. Karypis}.}
  \bibinfo{year}{2006}\natexlab{}.
\newblock \showarticletitle{{Coherent closed quasi-clique discovery from large
  dense graph databases}}. In \bibinfo{booktitle}{\emph{{SIGKDD}}}.
  \bibinfo{pages}{797--802}.
\newblock


\end{thebibliography}

\end{document}